\documentclass{article}

\usepackage{fullpage}
\usepackage[T1]{fontenc}    
\usepackage{hyperref}       
\usepackage{url}            
\usepackage{booktabs}       
\usepackage{amsfonts}       
\usepackage{nicefrac}       
\usepackage{microtype}      

\title{DNF-Net: A Neural Architecture for
Tabular Data}
\author{Ami Abutbul \\ amramabutbul@cs.technion.ac.il
   \and Gal Elidan \\ elidan@google.com
   \and Liran Katzir \\ lirank@google.com
   \and Ran El-Yaniv \\ rani@cs.technion.ac.il}

\usepackage{bbm}
\usepackage[numbers]{natbib}
\usepackage{amsmath}
\usepackage{xcolor}
\usepackage{graphicx}
\usepackage{amsthm}
\usepackage{amsfonts}
\usepackage{amssymb} 
\usepackage{wrapfig}
\usepackage[ruled,vlined]{algorithm2e}
\usepackage{booktabs}
\usepackage{placeins}
\usepackage{wrapfig}
\usepackage{subcaption}
\usepackage{multirow}
\usepackage{siunitx}
\usepackage{colortbl}
\usepackage{makecell}

\definecolor{yellow}{rgb}{0.976,0.839,0.41}
\definecolor{green}{rgb}{0.488,0.898,0.41}
\definecolor{blue}{rgb}{0.527,0.703,0.898}
\newcolumntype{Y}{>{\columncolor{yellow}}c}
\newcolumntype{G}{>{\columncolor{green}}c}
\newcolumntype{B}{>{\columncolor{blue}}c}

\newcommand{\eqdef}{\triangleq}

\newcommand{\paren}[1]{\kern-0.5ex\left( #1 \right)}

\newcommand{\sqparen}[1]{\left[ #1 \right]}

\newcommand{\abs}[1]{\left\lvert#1\right\rvert}
\newcommand{\sspm}[1]{\scriptsize ($\pm${#1})}

\newcommand{\sign}{\operatorname*{sign}}
\newcommand{\diag}{\operatorname*{diag}}
\newcommand{\loc}{\operatorname*{loc}}
\newcommand{\smloc}{\operatorname*{sm-loc}}
\newcommand{\Softmax}{\operatorname*{Softmax}}

\newcommand{\Ex}[2][]{
\ifx\hfuzz#1\hfuzz 
\mathbb{E}\sqparen{#2}
\else
\mathbb{E}_{#1}\sqparen{#2}
\fi
}

\newcommand{\reals}{\mathbb{R}}

\newcommand{\OR}{{\rm{OR}}}
\newcommand{\AND}{{\rm{AND}}}
\newcommand{\DNNF}{{\rm{DNNF}}}
\newcommand{\DNFNET}{{\rm{DNF\text{-}Net}}}
\newcommand{\DNFNETs}{{\rm{DNF\text{-}Nets }}}
\newcommand{\bx}{\mathbf{x}}
\newcommand{\bu}{\mathbf{u}}

\newcommand{\bb}{\mathbf{b}}
\newcommand{\bc}{\mathbf{c}}
\newcommand{\by}{\mathbf{y}}
\newcommand{\bm}{\mathbf{m}}
\newcommand{\bmu}{\boldsymbol\mu}
\newcommand{\bSigma}{\mathbf{\Sigma}}

\renewcommand{\bmu}{\mathbf{\mu}}

\newtheorem{lemma}{Lemma}
\newtheorem{theorem}{Theorem}
\newtheorem{definition}{Definition}

\begin{document}

\maketitle

\begin{abstract}

A challenging open question in deep learning is how to handle tabular data. Unlike domains such as image and natural language processing, where deep architectures prevail, there is still no widely accepted neural architecture that dominates tabular data. 
As a step toward bridging this gap, we present DNF-Net a novel generic architecture whose inductive bias elicits models whose structure corresponds to logical Boolean formulas in disjunctive normal form (DNF) over affine soft-threshold decision terms. In addition, DNF-Net promotes localized decisions that are taken over small subsets of the features. We present an extensive empirical study showing that DNF-Nets significantly and consistently outperform FCNs over tabular data.
With relatively few hyperparameters, DNF-Nets open the door to practical end-to-end handling of tabular data using neural networks.
We present ablation studies, which justify the design choices of DNF-Net including the three inductive bias elements, namely, Boolean formulation, locality, and feature selection. 
\end{abstract}

\section{Introduction}

A key point in successfully applying deep neural models is the construction of architecture families that contain inductive bias relevant to the application domain.
Architectures such as CNNs and RNNs have become the preeminent favorites for modeling images and sequential data, respectively. 
For example, the inductive bias of CNNs favors locality, as well as 
translation and scale invariances. With these properties, CNNs work extremely well
on image data, and are capable of generating problem-dependent representations that almost completely overcome the need for expert knowledge.
Similarly, the inductive bias promoted by RNNs and LSTMs (and more recent models such as transformers) favors both locality and temporal stationarity. 

When considering \emph{tabular data}, however, neural networks are not the hypothesis class of choice. Most often, the winning class in learning problems involving \emph{tabular data} is decision forests. In Kaggle competitions, for example, gradient boosting of decision trees (GBDTs) \citep{chen2016xgboost,friedman2001greedy, prokhorenkova2018catboost, ke2017lightgbm} are generally the superior model.
While it is quite practical to use GBDTs for medium size datasets, it is extremely hard to scale these methods to very large datasets (e.g., Google or Facebook scale).
The most significant computational disadvantage of GBDTs is the need to store (almost) the entire dataset in memory\footnote{This disadvantage is shared among popular GBDT implementations: XGBoost, LightGBM, and CatBoost.}. 
Moreover, handling multi-modal data, which involves both tabular and spatial data (e.g., images), is problematic. Thus, since GBDTs and neural networks cannot be organically optimized, such multi-modal tasks are left with sub-optimal solutions. Creating a purely neural model for tabular data, which can be trained with SGD end-to-end, is therefore a prime objective.


A few works have aimed at constructing neural models for tabular data (see Section~\ref{sec:related}). 
Currently, however, there is still no widely accepted end-to-end neural architecture that can handle
tabular data and consistently replace fully-connected architectures, or better yet, replace GBDTs. 
Here we present $\DNFNET$s, a family of neural network architectures whose primary inductive bias is an ensemble comprising a disjunctive normal form (DNF) formulas over linear separators. 
This family also promotes (input) feature selection and spatial localization of ensemble members.
These inductive biases have been included by design to promote 
conceptually similar elements that are inherent GBDTs and random forests.
Appealingly, the $\DNFNET$ architecture can be trained end-to-end using standard gradient-based optimization. Importantly, it consistently and significantly outperforms FCNs on tabular data, and can sometime even outperform GBDTs.



The choice of appropriate inductive bias for specialized hypothesis classes for tabular data is challenging since, clearly, there are many different kinds of such data. Nevertheless, 
the ``universality'' of forest methods in handling a wide variety of tabular data suggests that it might be beneficial to emulate, using neural networks, the important elements that are part of the tree ensemble representation and algorithms. Concretely, every decision tree is equivalent to some DNF formula over axis-aligned linear separators (see details in Section~\ref{sec:theoretical}). This makes DNFs an essential element in any such construction. Secondly, all contemporary forest ensemble methods rely heavily on feature selection. This feature selection is manifested both during the induction of each individual tree, where features are sequentially and greedily selected using information gain or other related heuristics, and by uniform sampling features for each ensemble member.
Finally, forest methods include an important localization element -- GBDTs with their sequential construction within a boosting approach, where each tree re-weights the instance domain differently -- and random forests with their reliance on bootstrap sampling.
$\DNFNET$s are designed to include precisely these three elements.

After introducing $\DNFNET$, we include a Vapnik-Chervonenkins (VC) comparative analysis of DNFs and trees showing that DNFs potentially have advantage over trees when the input dimension is large and vice versa. We then present an extensive empirical study. First, we present an ablation study over three real-life tabular data prediction tasks that convincingly demonstrates the importance of all three elements included in the $\DNFNET$ design. Second, we  analyze our novel feature selection component over controlled synthetic experiments, which indicate that this component is of independent interest.
Finally, we compare $\DNFNET$s to FCNs and GBDTs over several large classification tasks, including two past Kaggle competitions. Our results indicate that $\DNFNET$s consistently outperform FCNs, and can sometime even outperform GBDTs.



\section{Disjunctive Normal Form Networks (DNF-Nets)}
In this section we introduce the \DNFNET~architecture, which consists
of three elements. The main component is a block of layers emulating a DNF formula.
This block will be referred to as a \emph{Disjunctive Normal Neural Form} (DNNF). The second and third components, respectively, are a feature selection module, and a localization one. 
In the remainder of this section we describe each component in detail. Throughout our description we denote by $\bx \in \reals^d$ a column of input feature vectors, by $\bx_i$, its $i$th entry, and by $\sigma(\cdot)$ the sigmoid function.

\subsection{A Disjunctive Normal Neural Form (DNNF) Block}
\label{sec:DNF_structure}
A \emph{disjunctive normal neural form} (DNNF) block is assembled
using a two-hidden-layer network. The first layer creates
affine ``literals'' (features) and is trainable. The second layer implements a number of soft conjunctions over the literals, and the third output layer is a neural OR gate. Importantly, only the first layer is trainable, while the two other are binary and fixed.

We begin by describing the neural AND and OR gates. For an input vector $\bx$,
we define soft, differentiable versions of such gates as
\begin{eqnarray*}
\OR(\bx) &\eqdef& \tanh\paren{\sum_{i=1}^{d}\bx_i + d - 1.5 }, \hspace{50pt}
\AND(\bx) \eqdef \tanh\paren{\sum_{i=1}^{d}\bx_i - d + 1.5 }.
\end{eqnarray*}
These definitions are straightforwardly motivated by 
the precise neural implementation of the corresponding binary gates.
Notice that by replacing $\tanh$ by a binary activation and changing the bias constant 
from 1.5 to 1, we obtain an exact implementation of the corresponding logical gates
for binary input vectors \citep{anthony2005connections,shalev2014understanding};
see a proof of this statement in Appendix~\ref{sec:orand}.
Notably, each unit does not have any trainable parameters.
We now define the AND gate in a vector form to project the logical operation over a subset of variables. 
The projection is controlled by an indicator column vector (a mask)
$\bu \in \{0,1\}^d$. With respect to such a projection vector $\bu$, we define the corresponding \emph{projected} gate as $\AND_{\bu}(\bx) \eqdef \tanh\paren{\bu^T\bx - ||\bu||_1 + 1.5 }.$

Equipped with these definitions, a $\DNNF(\bx): \reals^d \rightarrow \reals$ with $k$ conjunctions over $m$ literals is,
\begin{align}
    L(\bx) & \eqdef \tanh\paren{\bx^TW + \bb} \in \reals^m \label{eq:literals} \\ 
    \DNNF(\bx) & \eqdef \OR\paren{[\AND_{\bc^1}(L(\bx)), 
    \AND_{\bc^2}(L(\bx)), \dots , \AND_{\bc^k}(L(\bx))]}  \label{eq:first_layer} .
\end{align}
Equation~(\ref{eq:literals}) defines $L(\bx)$ that generates 
$m$ ``neural literals'', each of which is the result of a $\tanh$-activation of a (trainable) affine transformation. 
The (trainable) matrix $W \in \reals^{d \times m}$, as well as the row vector bias term $\bb \in \reals^{m}$,
determine the affine transformations for each literal such that each of its columns corresponds to one literal.
Equation~(\ref{eq:first_layer}) defines a DNNF. In this equation, the vectors $\bc^i \in \{0,1\}^m$, $1 \leq i \leq k$, are binary indicators such that $\bc^i_j = 1$ iff the $j$th literal belongs to the $i$th conjunction. In our design, each literal belongs to a single conjunction. These indicator vectors are  defined and fixed according to the number and length of the conjunctions (See Appendix \ref{appendix:Training details}).

\subsection{DNF-Nets}

The embedding layer of a $\DNFNET$ with $n$ DNNF blocks is a simple concatenation
\begin{equation}
\label{eq:embedding}
E(\bx) \eqdef [\DNNF_1(\bx), \DNNF_2(\bx), \ldots, \DNNF_n(\bx)].
\end{equation}
Depending on the application, the final 
$\DNFNET$ is a composition of an output layer over $E(\bx)$. For example, for binary classification (logistic output layer),
$\DNFNET(x): \reals^d \rightarrow (0,1)$ has the following form, 
\begin{equation}
\label{eq:DNFNET}
    \DNFNET(\bx) \eqdef \sigma\left(\sum_{i=1}^n w_i \DNNF_i(\bx) + b_i\right) .
\end{equation}
To summarize, a $\DNFNET$ is always a four-layer network (including the output layer),
and only the first and last layers are learned.
Each DNNF block has two parameters: the number of conjunctions $k$ and the length $m$ of these conjunctions, allowing for a variety of $\DNFNET$ architectures.
In all our experiments we considered a single $\DNFNET$ architecture that has a
fixed diversity of DNNF blocks which includes a number of different DNNF groups with different $k$, each of which has a number of conjunction sizes $m$ (see details in Appendix~\ref{appendix:Training details}).
The number $n$ of DNNFs was treated as a hyperparameter, and selected based on a validation set as described on Appendix \ref{appendix:data partition and grid search process appendix}.



\subsection{Feature Selection}
\label{sec:feature_selection}
One key strategy in decision tree training is greedy feature selection,
which is performed hierarchically at any split, and allows decision trees to exclude irrelevant features. Additionally, 
decision tree ensemble algorithms apply random sampling to select a subset of the features, which is used to promote diversity, and prevent different trees focusing on the same set of dominant features in their greedy selection.
In line with these strategies, we include in our $\DNFNET$s 
conceptually similar feature selection elements:
(1) a subset of features uniformly and randomly sampled for each \DNNF;
(2) a trainable mechanism for feature selection, applied on the resulting random subset. These two elements are combined and implemented in the affine literal generation layer described in Equation~(\ref{eq:literals}), and applied independently for each \DNNF.
We now describe these techniques in detail.

Recalling that $d$ is the input dimension, the random selection is made by generating a stochastic binary mask, $\bm_s \in \{0, 1\}^d$, such that the probability of any entry being 1 is $p$ (see Appendix \ref{appendix:Training details} for details on setting this parameter). For a given mask $\bm_s$, this selection can be applied over affine literals using 
a simple product $\diag(\bm_s)W$, where $W$ is the matrix of Equation~(\ref{eq:literals}). 
We then construct a \emph{trainable} mask $\bm_t \in \reals^d$, which will be applied on the features that are kept by $\bm_s$. 
We introduce a novel trainable feature selection component that combines binary quantization of the mask together with modified elastic-net regularization.
To train a binarized 
vector we resort to the straight-through estimator \cite{HintonLectures,hubara2017quantized}, which can be used effectively to train 
non-differentiable step functions such as a threshold or sign. The trick is to compute  
the step function \emph{exactly} in the forward pass, and utilize a differentiable proxy in the  backward pass.
We use a version of the straight-through estimator for the sign function \cite{bengio2013estimating}, 
$$
        \Phi(x) \eqdef \begin{cases}
                \sign(x), & \text{forward pass};\\
                \tanh(x), & \text{backward pass} .\\
                \end{cases}
$$
Using the estimator $\Phi(x)$, we define a differentiable binary threshold 
function $T(x) = \frac{1}{2}\Phi(|x| - \epsilon) + \frac{1}{2}$, where $\epsilon \in \reals$ defines an epsilon neighborhood around zero for which the output of $T(x)$ is zero, and one outside of this neighborhood (in all our experiments, we set $\epsilon=1$ and initialize the entries of $\bm_t$ above this threshold). We then apply this selection by $\diag(T(\bm_t))W$. 


Given a fixed stochastic selection $\bm_s$, to train the binarized selection $\bm_t$
we employ regularization. Specifically, we consider a modified version of the elastic net regularization, $R(\bm_t,\bm_s)$, which is tailored to our task. The modifications are reflected in two parts. First, the balancing between the $L_1$ and $L_2$ regularization is controlled by a trainable parameter $\alpha \in \reals$. Second, the expressions of the $L_1$ and $L_2$ regularization are replaced by $R_1(\bm_t,\bm_s), R_2(\bm_t,\bm_s)$, respectively (defined below). Moreover, since we want to take into account only features that were selected by the random component, the regularization is applied on the vector $\bm_{ts} = \bm_t \odot \bm_s$, where $\odot$ is element-wise multiplication. 
The functional form of the modified elastic net regularization is as follows,
\begin{gather*}
R_2(\bm_t,\bm_s) \eqdef \abs{\frac{||\bm_{ts}||_2^2}{||\bm_s||_1} - \beta\epsilon^2}, \hspace{50pt}
R_1(\bm_t,\bm_s) \eqdef \abs{\frac{||\bm_{ts}||_1}{||\bm_s||_1} - \beta\epsilon} \\
R(\bm_t,\bm_s) \eqdef  \frac{1-\sigma(\alpha)}{2} R_2(\bm_t,\bm_s) + \sigma(\alpha) R_1(\bm_t,\bm_s).
\end{gather*}
The above formulation of $R_2(\cdot)$ and $R_1(\cdot)$ is motivated as follows.
First, we normalize both norms by dividing with the effective input dimension, $||\bm_s||_1$, 
which is done to be invariant to the (effective) input size. Second, we define $R_2$ and $R_1$ as 
absolute errors, which encourages each entry to be, on average, approximately equal to the threshold $\epsilon$. 
The reason is that the vector $\bm_t$ passes through a binary threshold, and though the exact values of its entries are irrelevant. What is relevant is whether these values are within epsilon neighborhood of zero or not. Thus, when the values are roughly equal to the threshold, it is more likely to converge to a balanced point where the regularization term is low and the relevant features were selected. The threshold term is controlled by $\beta$ (a hyperparameter), which controls the cardinality of $\bm_t$, where smaller values of $\beta$ lead to sparser $\bm_t$.



Finally, the functional form of a DNNF block with the feature selection component is obtained by plugging the masks into Equation~(\ref{eq:first_layer}), $L = \tanh\paren{\bx^T \diag(T(\bm_t))\diag(\bm_s)W + \bb} \in \reals^m$.
Additionally, the mean over $R(\bm_t,\bm_s)$ in all \DNNF s is added to the loss function as a regularizer.

\subsection{Spacial Localization}
\label{sec:localization}
The last element we incorporate in the $\DNFNET$ construction
is \emph{spatial localization}. This element encourages
each DNNF unit in a $\DNFNET$ ensemble  
to specialize in some focused proximity of the input domain.
Localization is a well-known technique in classical machine learning, with various implementations and applications (see, e.g., \cite{jacobs1991adaptive,meir2000localized}). 
On the one hand, localization allows construction of low-bias experts. On the other hand, it helps promote diversity, and reduction of the correlation between experts, which can improve the performance of an ensemble \cite{jacobs1997bias,derbeko2002variance}.

We incorporate spatial localization by associating a Gaussian kernel $\loc(\bx | \bmu, \bSigma)_i$ 
with a trainable mean vector $\bmu_i$ and a trainable diagonal covariance matrix $\bSigma_i$ for
the $i$th DNNF. Given a $\DNFNET$ with $n$ DNNF blocks, the functional form of its embedding layer (Equation \ref{eq:embedding}), with the spatial localization, is
\begin{align*}
    \loc(\bx | \bmu, \bSigma) &\eqdef [e^{-||\bSigma_1(\bx-\bmu_1)||_2}, e^{-||\bSigma_2(\bx-\bmu_2)||_2}, \dots, e^{-||\bSigma_n(\bx-\bmu_n)||_2}] \in \reals^n \\
    \smloc(\bx | \bmu, \bSigma) &\eqdef \Softmax \left\{ \loc(\bx | \bmu, \bSigma) \cdot \sigma(\tau) \right\} \in (0, 1)^n\\
    E(\bx) &\eqdef [\smloc(\bx | \bmu, \bSigma)_1 \cdot \DNNF_1(\bx), \ldots, \smloc(\bx | \bmu, \bSigma)_n \cdot \DNNF_n(\bx)],
\end{align*}
where $\tau \in \reals$ is a trainable parameter such that $\sigma(\tau)$ 
serves as the trainable temperature in the softmax.
The inclusion of an adaptive temperature in this localization mechanism
facilitates a data-dependent degree of exclusivity: at high temperatures,
only a few DNNFs will handle an input instance whereas at low temperatures, 
more DNNFs will effectively participate in the ensemble.
Observe that our localization mechanism is fully trainable and does not add any
hyperparameters.


\section{DNFs and Trees -- A VC Analysis}
\label{sec:theoretical}

The basic unit in our construction is a (soft) DNF formula instead of a tree. Here we provide a theoretical perspective on this design choice. Specifically, we analyze the VC-dimension of 
Boolean DNF formulas and compare it to that of decision trees. With this analysis we gain some insight into the generalization ability of formulas 
and trees, and argue numerically that the generalization of a DNF can be superior to a tree when the input dimension is not small (and vice versa).

Throughout this discussion, we consider binary classification problems
whose instances are Boolean vectors in $\{0,1\}^n$.
The first simple observation is that
every decision tree has an equivalent DNF formula. Simply,
each tree path from the root to a positively labeled leaf can be expressed 
by a conjunction of the conditions over the features appearing along the path to the leaf, and the whole tree can be represented by a disjunction of the resulting conjunctions.
However, DNFs and decision trees are not equivalent,
and we demonstrate that in the lense of VC-dimension.
 Simon et al. \cite{Simon} presented an exact expression for the VC-dimension of decision trees as a function of the tree \emph{rank}. 

\begin{definition}[Rank]
Consider a binary tree $T$. 
If $T$ consists of a single node, its rank is defined as 0. 
If $T$ consists of a root, a left subtree $T_0$ of rank $r_0$, and a right subtree $T_1$ of rank $r_1$, then   
\begin{equation*}
    rank(T)=\begin{cases}
			1 + r_0 & \text{if $r_0$ = $r_1$} \\
            \max\{r_0, r_1\} & \text{else}
		 \end{cases}
\end{equation*}
\end{definition}

\begin{wrapfigure}{R}{0.4\textwidth}
 \vspace{-0.15in}
    \centering
    \includegraphics[width=0.4\textwidth]{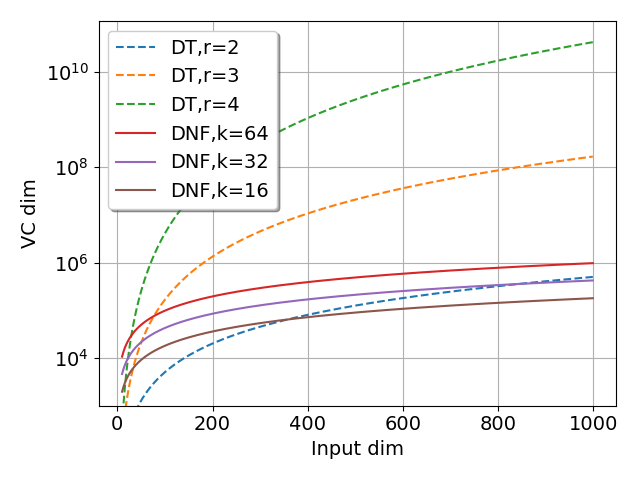}
    \caption{$VCDim(DT_n^r)$ and the upper bound on $VCDim(DNF_n^k)$ (log scale) as a function of the input dimension}
    \vspace{-0.4in}
    \label{fig:VCDim}
\end{wrapfigure}

Clearly, for any decision tree $T$ over $n$
variables, $1 \leq rank(T) \leq n$.
Also, it is not hard to see that a binary tree $T$ has a rank greater than  $r$ iff the complete binary tree of depth $r+1$ can be embedded into $T$. 

\begin{theorem}[Simon, \cite{Simon}]
\label{theorem:VCdimDT}
Let $DT_n^r$ denote the class of decision trees of rank at most $r$ on $n$ Boolean variables. Then it holds that
$
    VCDim(DT_n^r) = \sum_{i=0}^r \binom{n}{i} .
$
\end{theorem}
The following theorem, whose proof appears in  Appendix~\ref{sec:proof_theorem1},
upper bounds the VC-dimension of a Boolean DNF formula.

\begin{theorem}[DNF VC-dimension bound]
\label{theorem:VCdimDNF}
Let $DNF_n^k$ be the class of DNF formulas with $k$ conjunctions on $n$ Boolean variables. Then it holds that
$
VCDim(DNF_n^k) \leq 2(n+1)k\log(3k) .
$
\end{theorem}

It is evident that in the case of DNF formulas the upper bound on 
the VC-dimension grows linearly with the input dimension, whereas in the case of decision trees, if the rank is greater than 1, the VC-dimension
grows polynomially (with degree at least 2) with the input dimension. In the worst case, this growth is exponential.
A direct comparison of these dimensions is not trivial because there is a complex dependency between the rank $r$ of a decision tree, and the number $k$ of the conjunctions of an equivalent DNF formula. Even if we compare 
large-$k$ DNF formulas to small-rank trees, it is 
clear that the VC-dimension of the trees can be significantly larger. For example, in Figure~\ref{fig:VCDim}, we plot the upper bounds on the VC-dimension of large formulas (solid
curves), and the exact VC-dimensions of small-rank trees 
(dashed curves). With the exception of rank-2 trees, 
the VC-dimension of decision trees dominates the dimension of DNFs,
when the input dimension exceeds 100.
Trees, however, may have an advantage over DNF formulas for low-dimensional inputs.

Since the VC-dimension is a qualitative proxy of the sample complexity of a hypothesis class, the above analysis provides theoretical motivation for expressing trees 
using DNF formulas when the input dimension is not small.
Having said that, the disclaimer is that in the present discussion we have only considered binary problems. Moreover, the final hypothesis classes of both DNF-Nets 
and GBDTs are more complex in structure.


\section{Empirical Study}
In this section, we present an empirical study that substantiates the design of $\DNFNET$s and
convincingly shows its significant advantage over FCN architectures.
The datasets used in this study are from Kaggle competitions 
and OpenML
\citep{vanschoren2014openml}. A summary of these datasets appears in Appendix \ref{sec:Appendix_datasets_description}.
All results presented in this work were obtained using a massive grid search for optimizing each model's hyperparameters. A detailed description of the grid search process with additional details can be found in Appendices \ref{appendix:data partition and grid search process appendix}, \ref{appendix:Training details}. We present the scores for each dataset according to the score function defined in the Kaggle competition we used, log-loss and area under ROC curve (AUC ROC) for multiclass datasets and binary datasets, respectively. All results are the mean of the test scores over five different partitions, and the standard error of the mean is reported.

{\bf The merit of the different DNF-Net components.}
We start with two different ablation studies, where we evaluate the contributions
of the three \DNFNET~components.
In the first study, we start with a vanilla three-hidden-layer FCN and gradually add each component separately. In the second study, we start each experiment with the complete $\DNFNET$ and leave one component out each time. In each study, we present the results on three
real-world datasets, where all results are test log-loss scores (lower is better), out-of-memory (OOM) entries mean that the network was too 
large to execute on our machine (see Appendix~\ref{appendix:Training details}).
More technical details can be found in Appendix \ref{appendix:Ablation Studies - technical details}.

\begin{table}[h]
\centering
\caption{\textbf{Gradual study} (test log-loss scores)}
 \scalebox{0.85}{
  \begin{tabular}{lYYYGGGBBB}
    \toprule
    \multirow{1}{*}{\textbf{Dataset}} &
      \multicolumn{3}{>{\columncolor{yellow}}c}{\textbf{Eye Movements}} &
      \multicolumn{3}{>{\columncolor{green}}c}{\textbf{Gesture Phase}} &
      \multicolumn{3}{>{\columncolor{blue}}c}{\textbf{Gas Concentrations}} \\
      \textbf{ \# formulas} & {\textbf{128}} & {\textbf{512}} & {\textbf{2048}} & {\textbf{128}} & {\textbf{512}} & {\textbf{2048}} & {\textbf{128}} & {\textbf{512}} & {\textbf{2048}} \\
      \midrule
    \thead{Exp 1: Fully \\ trained FCN} & \thead{0.9864 \\ {\sspm{0.0038}}} & \thead{1.0138 \\ {\sspm{0.0083}}} &	OOM &	\thead{1.3139 \\{\sspm{0.0067}}} &	\thead{1.3368 \\{\sspm{0.0084}}} &	OOM	& \thead{0.2423 \\{\sspm{0.0598}}} &	\thead{0.3862 \\{\sspm{0.0594}}} &	OOM \\
    
    \thead{Exp 2: Adding \\ DNF structure} & \thead{0.9034 \\{\sspm{0.0058}}} &	\thead{0.9336 \\{\sspm{0.0058}}} &	\thead{1.3011 \\{\sspm{0.0431}}} &	\thead{1.1391 \\{\sspm{0.0059}}} &	\thead{1.1812 \\{\sspm{0.0117}}} &	\thead{1.8633 \\{\sspm{0.1026}}} &	\thead{0.0351 \\{\sspm{0.0048}}} &	\thead{0.0421 \\{\sspm{0.0046}}} &	\thead{0.0778 \\{\sspm{0.0080}}} \\
    
    \thead{Exp 3: Adding \\ feature selection} &
    \thead{0.8134 \\{\sspm{0.0142}}}&	\thead{0.8163 \\{\sspm{0.0096}}}&	 \thead{0.9652 \\{\sspm{0.0143}}}&	\thead{1.1411 \\{\sspm{0.0093}}}&	\thead{1.1320 \\{\sspm{0.0083}}}&	\thead{1.3015 \\{\sspm{0.0317}}}&	\thead{0.0227 \\{\sspm{0.0019}}}&	\thead{0.0265 \\{\sspm{0.0012}}}&	\thead{0.0516 \\{\sspm{0.0061}}} \\

    \thead{Exp 4: Adding \\ localization} & \thead{0.7621 \\{\sspm{0.0079}}}&	\thead{0.7125 \\{\sspm{0.0077}}}&	\thead{\textbf{0.6903} \\{\sspm{0.0049}}}&	\thead{0.9742 \\{\sspm{0.0079}}}&	\thead{0.9120 \\{\sspm{0.0123}}}&	\thead{\textbf{0.8770} \\{\sspm{0.0088}}}&	\thead{0.0162 \\{\sspm{0.0013}}}&	\thead{0.0149 \\{\sspm{0.0008}}}&	\thead{\textbf{0.0145} \\{\sspm{0.0011}}} \\
        
    \bottomrule
  \end{tabular}}
  \label{tab:abl1}
\end{table}

Consider Table~\ref{tab:abl1}. In \textbf{Exp 1} we start with a vanilla three-hidden-layer FCN with a $\tanh$ activation. To make a fair comparison, we defined the widths of the layers according to the widths in the $\DNFNET$ with the corresponding formulas. In \textbf{Exp 2}, we added the DNF structure to the networks from Exp 1 (see Section \ref{sec:DNF_structure}). In \textbf{Exp 3} we added the feature selection component (Section \ref{sec:feature_selection}). 
In is evident that performance is monotonically improving, where the best results are clearly obtained on the complete $\DNFNET$ (Exp 4).
A subtle but important observation is that in all of the first three experiments, for all  datasets, the trend is that the lower the number of formulas, the better the score. This trend is reversed in \textbf{Exp 4}, where the localization component (Section \ref{sec:localization}) is added, highlighting the importance of using all components of the \DNFNET~representation in concert.

\begin{table}[h]
\centering
\caption{\textbf{Leave one out study} (test log-loss scores)}
 \scalebox{0.85}{
  \begin{tabular}{lYYYGGGBBB}
    \toprule
    \multirow{1}{*}{\textbf{Dataset}} &
      \multicolumn{3}{>{\columncolor{yellow}}c}{\textbf{Eye Movements}} &
      \multicolumn{3}{>{\columncolor{green}}c}{\textbf{Gesture Phase}} &
      \multicolumn{3}{>{\columncolor{blue}}c}{\textbf{Gas Concentrations}} \\
      \textbf{\# formulas} & {\textbf{128}} & {\textbf{512}} & {\textbf{2048}} & {\textbf{128}} & {\textbf{512}} & {\textbf{2048}} & {\textbf{128}} & {\textbf{512}} & {\textbf{2048}} \\
      \midrule
      
    \thead{Exp 4: Complete \\ \DNFNET} & \thead{0.7621 \\{\sspm{0.0079}}}&	\thead{0.7125 \\{\sspm{0.0077}}}&	 \thead{\textbf{0.6903} \\{\sspm{0.0049}}}&	\thead{0.9742 \\{\sspm{0.0079}}}&	\thead{0.9120 \\{\sspm{0.0123}}}&	 \thead{\textbf{0.8770} \\{\sspm{0.0088}}}&	\thead{0.0162 \\{\sspm{0.0013}}}&	\thead{0.0149 \\{\sspm{0.0008}}}&	\thead{\textbf{0.0145} \\{\sspm{0.0011}}} \\

    \thead{Exp 5: Leave \\ feature selection out} &
    \thead{0.8150 \\{\sspm{0.0046}}}&	\thead{0.8031 \\{\sspm{0.0046}}}&	 \thead{0.7969 \\{\sspm{0.0054}}}&	\thead{0.9732 \\{\sspm{0.0082}}}&	\thead{0.9479 \\{\sspm{0.0081}}}&	 \thead{0.9438 \\{\sspm{0.0111}}}&	\thead{0.0222 \\{\sspm{0.0018}}}&	\thead{0.0205 \\{\sspm{0.0021}}}&	\thead{0.0200 \\{\sspm{0.0022}}} \\

    \thead{Exp 6: Leave \\ localization out} & 
    \thead{0.8134 \\{\sspm{0.0142}}}&	\thead{0.8163 \\{\sspm{0.0096}}}&	 \thead{0.9652 \\{\sspm{0.0143}}}&	\thead{1.1411 \\{\sspm{0.0093}}}&	\thead{1.1320 \\{\sspm{0.0083}}}&	 \thead{1.3015 \\{\sspm{0.0317}}}&	\thead{0.0227 \\{\sspm{0.0019}}}&	\thead{0.0265 \\{\sspm{0.0012}}}&	\thead{0.0516 \\{\sspm{0.0061}}}\\

    \thead{Exp 7: Leave DNF \\ structure out} & \thead{0.8403 \\{\sspm{0.0068}}}&	\thead{0.8128 \\{\sspm{0.0077}}}& OOM &	\thead{1.1265 \\{\sspm{0.0066}}}&	\thead{1.1101 \\{\sspm{0.0077}}}& OOM &	\thead{0.0488 \\{\sspm{0.0038}}}&	\thead{0.0445 \\{\sspm{0.0024}}}& OOM \\
        
    \bottomrule
  \end{tabular}}
    \label{tab:abl2}
\end{table}

Now consider Table~\ref{tab:abl2}.
In \textbf{Exp 5} we took the complete \DNFNET~ (Exp 4) and removed the feature selection component. When considering the Gesture Phase dataset, an interesting phenomenon is observed. In  Exp 3 (128 formulas), we can see that the contribution of the feature selection component is negligible, but in Exp 5 (2048 formulas) we see the significant contribution of this component. We believe that the reason for this difference lies in the relationship of the feature selection component with the localization component, where this connection intensifies the contribution of the feature selection component. In \textbf{Exp 6} we took the complete \DNFNET~ (Exp 4) and removed the localization component (identical to Exp 3). We did the same in \textbf{Exp 7} where we removed the DNF structure. In general, it can be seen that removing each component results in a decrease in performance.

{\bf An analysis of the feature selection component.}
Having studied the contribution of the three
components to $\DNFNET$, we now focus 
on the learnable part of the feature selection component (Section~\ref{sec:feature_selection}) alone, and examine its effectiveness using a series of synthetic tasks with a varying percentage of irrelevant features. Recall that when considering a single DNNF
block, the feature selection is a learnable binary mask that multiplies the input element-wise. Here we examine the effect of this mask on a vanilla FCN network (see technical details in Appendix \ref{appendix:Feature selection analysis}).
The synthetic tasks we use were introduced by Yoon et al.\cite{yoon2018invase}, and Chen et al.\cite{chen2018learning}, where they were used as synthetic experiments to test feature selection. 
There are six different dataset settings; exact 
specifications appear in Appendix \ref{appendix:Feature selection analysis}. 
For each dataset, we generated seven different instances that differ in their input size. While increasing the input dimension $d$, the same logit is used for prediction, so the new features are irrelevant, and as $d$ gets larger, the percentage of relevant features becomes smaller.

We compare the performance of a vanilla FCN on three different cases: (1) oracle (ideal) feature selection (2) our (learned) feature selection mask, and (3) no feature selection. (See  details in Appendix \ref{appendix:Feature selection analysis}). Consider the graphs in Figure~\ref{fig:synth},
which demonstrate several interesting insights.
In all tasks the performance of the vanilla FCN is sensitive to irrelevant features, probably due to the representation power of the FCN, which is prone to overfitting.
On the other hand, by adding the feature selection component, we obtain near oracle performance on the first three tasks, and a significant improvement on the three others. Moreover, these results support our observation from the ablation studies: that the application of localization  together with feature selection increases the latter's contribution. We can see that in Syn1-3 where there is a single interaction, the results are better than in Syn4-6 where the input space is divided into two `local' sub-spaces with different interactions.
These experiments indicate that the learnable feature selection we propose can have independent value.

\begin{table}[h!]
\begin{center}
 \scalebox{0.95}{
 \begin{tabular}{||l c c c c||} 
 \hline
 Dataset & Test Metric  & \DNFNET & XGBoost & FCN \\ [0.5ex] 
 \hline\hline
     Otto Group             & log-loss      & $\mathbf{45.600 \pm 0.445}$     & $45.705 \pm 0.361$             & $47.898 \pm 0.480$\\ 
 \hline
    Gesture Phase           & log-loss      & $86.798 \pm 0.810$             & $\mathbf{81.408 \pm 0.806}$    & $102.070 \pm 0.964$\\ 
 \hline
    Gas Concentrations      & log-loss      & $\mathbf{1.425 \pm 0.104}$     & $2.219 \pm 0.219$              & $5.814 \pm 1.079$\\ 
 \hline
    Eye Movements           & log-loss      & $68.037 \pm 0.651$             & $\mathbf{57.447 \pm 0.664}$    & $78.797 \pm 0.674$\\ 
 \hline
   Santander Transaction   & roc auc       & $88.668 \pm 0.128$             & $\mathbf{89.682 \pm 0.165}$    & $86.722 \pm 0.158$\\
 \hline
    House                   & roc auc       & $95.451 \pm 0.092$             & $\mathbf{95.525 \pm 0.138}$    & $95.164 \pm 0.103$\\ 
 \hline

\end{tabular}}
\end{center}
\caption{Mean test results on tabular datasets and standard error of the mean. We present the ROC AUC (higher is better) as a percentage, and the log-loss (lower is better) with an x100 factor.}
\label{tab:tabular_results}
\end{table}

{\bf Comparative Evaluation. }
Finally, we compare the performance of $\DNFNET$
vs. the baselines.  Consider Table~\ref{tab:tabular_results}
where we examine the performance of $\DNFNET$s on six real-life tabular datasets (We add three larger datasets to those we used in the ablation studies). We compare our performance to XGboost \cite{chen2016xgboost}, the widely used implementation of GBDTs, and to FCNs.
For each model, we optimized its critical hyperparameters. This optimization process required many computational resources: thousands of configurations have been tested for FCNs, hundreds of configurations for XGBoost, and only a few dozen for $\DNFNET$. A detailed description of the grid search we used for each model can be found in Appendix \ref{appendix:grid_params}.
In Table 3, we see that $\DNFNET$ consistently and significantly outperforms FCN over all the six datasets. While obtaining better than or indistinguishable results from XGBoost over two datasets, on the other datasets, $\DNFNET$ is slightly inferior but in the same ball park as XGBoost.

\begin{figure}[th]
\begin{subfigure}{0.33\textwidth}
\includegraphics[width=0.9\linewidth]{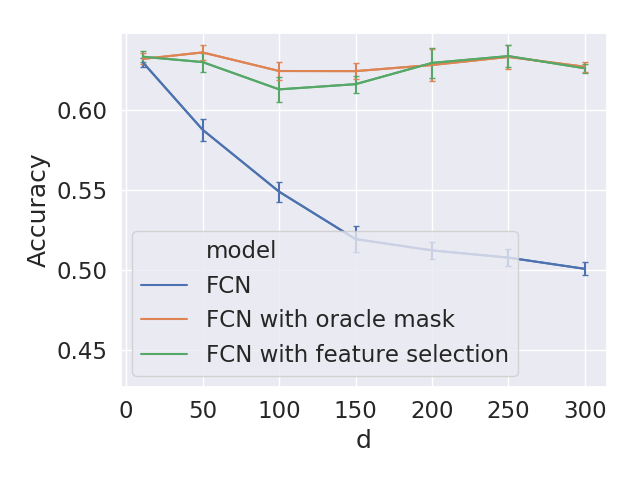}
\caption{Syn1}
\end{subfigure}
\begin{subfigure}{0.33\textwidth}
\includegraphics[width=0.9\linewidth]{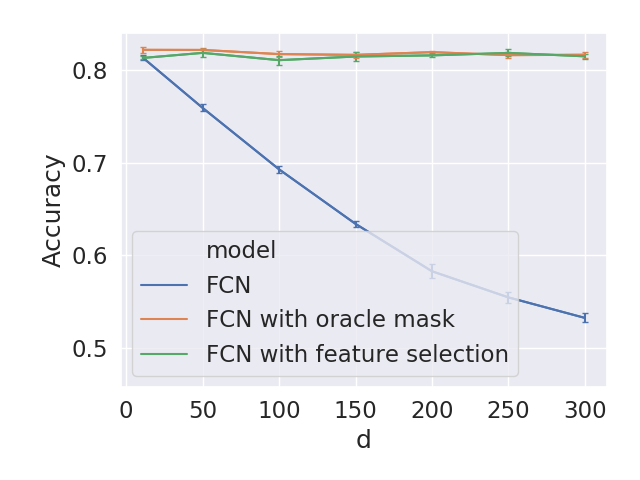}
\caption{Syn2}
\end{subfigure}
\begin{subfigure}{0.33\textwidth}
\includegraphics[width=0.9\linewidth]{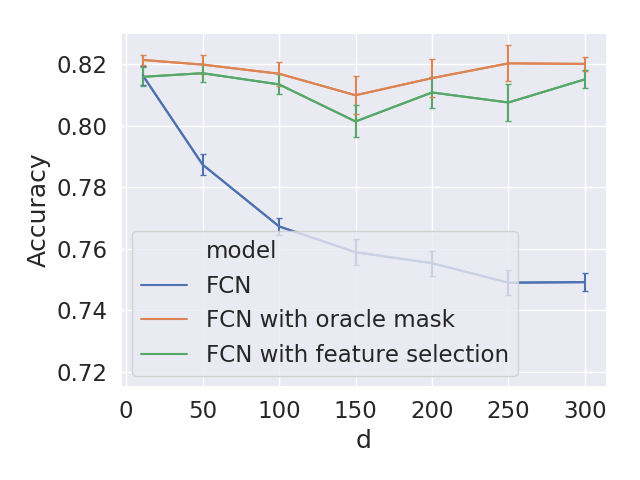}
\caption{Syn3}
\end{subfigure}

\begin{subfigure}{0.33\textwidth}
\includegraphics[width=0.9\linewidth]{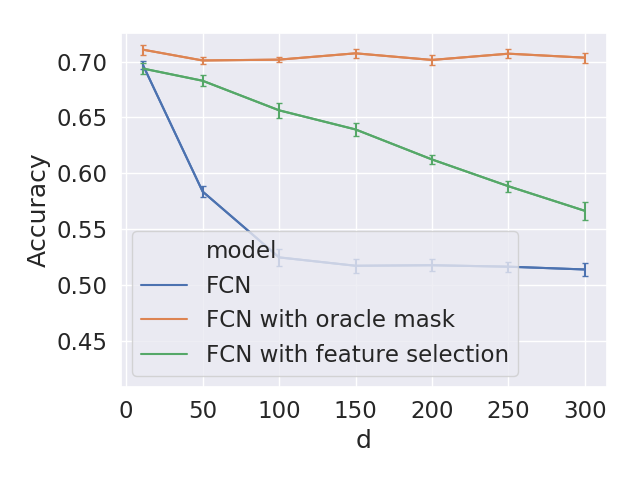}
\caption{Syn4}
\end{subfigure}
\begin{subfigure}{0.33\textwidth}
\includegraphics[width=0.9\linewidth]{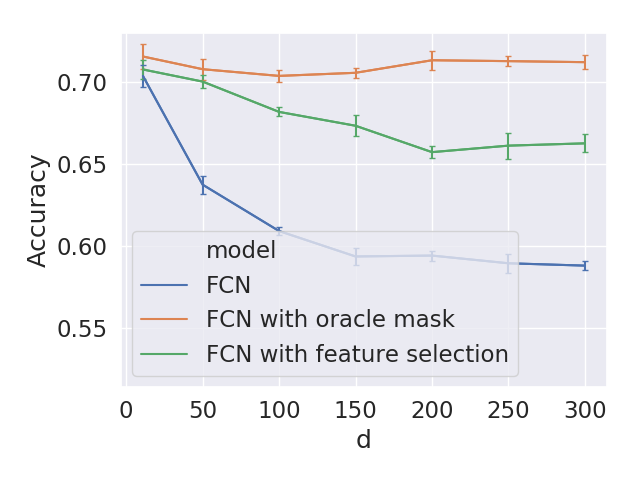}
\caption{Syn5}
\end{subfigure}
\begin{subfigure}{0.33\textwidth}
\includegraphics[width=0.9\linewidth]{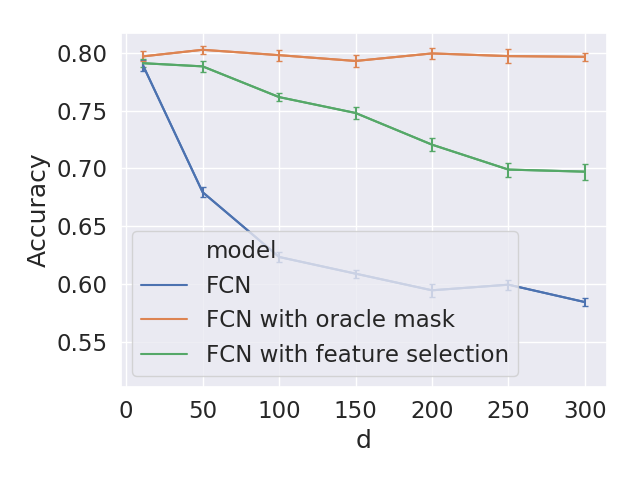}
\caption{Syn6}
\end{subfigure}

\caption{The results on the six synthetic experiments. For each experiment we present the test accuracy (with an error bar of the standard error of the mean) as a function of the input dimension $d$.}
\label{fig:synth}
\end{figure}


\section{Related Work}
\label{sec:related}
There have been  a few attempts to construct neural networks with improved performance on tabular data.
A recurring idea in some of these works is the explicit use of conventional decision tree induction algorithms, such as ID3 \citep{quinlan1979discovering}, or conventional forest methods, such as GBDT \citep{friedman2001greedy} that are 
trained over the data at hand, and then parameters of the resulting decision trees are explicitly or implicitly ``imported'' into a neural network using teacher-student distillation \citep{ke2018tabnn}, 
explicit embedding of tree paths in a specialized network architecture with some kind of DNF structure \citep{seyedhosseini2015disjunctive},
and explicit utilization of forests as the main building block of layers \citep{feng2018multi}.
This reliance on conventional decision tree or forest methods as an integral part of the proposed solution prevents end-to-end neural optimization, as we propose here. This deficiency is not only a theoretical nuisance but also makes it hard to use such models on very large datasets and in combination with other neural modules.

A few other recent techniques aimed to cope with tabular data using pure neural optimization as we propose here.   
\cite{yang2018deep} considered a method to approximate a single node of a decision tree using a soft binning function that transforms continuous features into one-hot features. 
While this method obtained results comparable to a single decision tree and an FCN (with two hidden layers), it is limited to settings where the number of features is small.
Popov et al. \cite{popov2019neural} proposed a network that combines elements of oblivious decision forests with dense residual networks. While this method achieved better results than GBDTs on several datasets, also FCNs 
achieved better than or indistinguishable results from GBDTs on most of these cases as well. 
Finally, focusing on microbiome data, a recent study \cite{shavitt2018regularization} presented an elegant regularization technique, which produces extremely sparse networks that are suitable for microbiome tabular datasets.

Soft masks for feature selection have been considered before and 
the advantage of using elastic net regularization in a variable selection task was presented by Zou and Hastie and others \cite{zou2005regularization,li2016deep}.

\section{Conclusions}
We introduced \DNFNET, a novel neural architecture whose inductive bias 
revolves around a disjunctive normal neural form, localization and feature selection. 
The importance of each of these elements has been demonstrated over real tabular data.
The results of the empirical study indicate convincingly that \DNFNETs 
consistently outperform FCNs over tabular data. While \DNFNETs do not consistently 
beat XGBoost, our results indicate that their performance score is not far behind.

We have left a number of potential incremental improvements and bigger challenges 
to future work. First, in our work we only considered classification problems.
We expect \DNFNETs to also be effective in regression problems, and it would also 
be interesting to consider 
applications in reinforcement learning over finite discrete spaces. 
It would be very interesting to consider deeper \DNFNET~architectures. For example,
instead of a single DNNF block, one can construct a stack of such blocks
to allow for more involved feature generation. 
Another interesting direction would be to consider training \DNFNETs using a 
gradient boosting procedure similar to that used in XGBoost.

Finally, a most interesting challenge that remains open is what would constitute 
a usable and effective inductive bias for tabular prediction tasks, which can 
elicit the best architectural designs for these data. 
Our successful application of DNNFs indicates that soft DNF formulas are 
quite effective, and are strictly significantly superior to fully connected networks, 
but we anticipate that further biases will be identified, at least for some families
of tabular tasks.

\section*{Broader Impact}
In this paper we present a new family of neural 
architectures for tabular data. Since much of the medical information on people has multiple modalities including a tabular form,
the societal effect and potential positive outcomes of this work are quite substantial, by contributing to our
ability to handle multi-modal data end-to-end using neural networks.
Negative consequences might appear, of course, if 
agents will utilize this technology to handle large scale tabular data to achieve some malicious 
objectives. However, there is no particular malicious application foreseen.

\bibliographystyle{plain}
\bibliography{references}

\begin{thebibliography}{10}

\bibitem{abadi2016tensorflow}
Martin Abadi, Paul Barham, Jianmin Chen, Zhifeng Chen, Andy Davis, Jeffrey
  Dean, Matthieu Devin, Sanjay Ghemawat, Geoffrey Irving, Michael Isard, et~al.
\newblock Tensorflow: A system for large-scale machine learning.
\newblock In {\em 12th $\{$USENIX$\}$ Symposium on Operating Systems Design and
  Implementation ($\{$OSDI$\}$ 16)}, pages 265--283, 2016.

\bibitem{anthony2005connections}
Martin Anthony.
\newblock Connections between neural networks and {Boolean} functions.
\newblock In {\em Boolean Methods and Models}, 2005.

\bibitem{bengio2013estimating}
Yoshua Bengio, Nicholas Leonard, and Aaron Courville.
\newblock Estimating or propagating gradients through stochastic neurons for
  conditional computation.
\newblock {\em arXiv preprint arXiv:1308.3432}, 2013.

\bibitem{blumer1989learnability}
Anselm Blumer, Andrzej Ehrenfeucht, David Haussler, and Manfred~K Warmuth.
\newblock Learnability and the vapnik-chervonenkis dimension.
\newblock {\em Journal of the ACM (JACM)}, 36(4):929--965, 1989.

\bibitem{chen2018learning}
Jianbo Chen, Le~Song, Martin~J Wainwright, and Michael~I Jordan.
\newblock Learning to explain: An information-theoretic perspective on model
  interpretation.
\newblock {\em arXiv preprint arXiv:1802.07814}, 2018.

\bibitem{chen2016xgboost}
Tianqi Chen and Carlos Guestrin.
\newblock Xgboost: A scalable tree boosting system.
\newblock In {\em Proceedings of the 22nd ACM SIGKDD International Conference
  on Knowledge Discovery and Data Mining}, pages 785--794. ACM, 2016.

\bibitem{derbeko2002variance}
Philip Derbeko, Ran El-Yaniv, and Ron Meir.
\newblock Variance optimized bagging.
\newblock In {\em European Conference on Machine Learning}, pages 60--72.
  Springer, 2002.

\bibitem{feng2018multi}
Ji~Feng, Yang Yu, and Zhi-Hua Zhou.
\newblock Multi-layered gradient boosting decision trees.
\newblock In {\em Advances in Neural Information Processing Systems}, pages
  3555--3565, 2018.

\bibitem{friedman2001greedy}
Jerome~H Friedman.
\newblock Greedy function approximation: a gradient boosting machine.
\newblock {\em Annals of Statistics}, pages 1189--1232, 2001.

\bibitem{HintonLectures}
Geoffrey Hinton.
\newblock Neural networks for machine learning coursera video lectures -
  geoffrey hinton.
\newblock 2012.

\bibitem{hubara2017quantized}
Itay Hubara, Matthieu Courbariaux, Daniel Soudry, Ran El-Yaniv, and Yoshua
  Bengio.
\newblock Quantized neural networks: Training neural networks with low
  precision weights and activations.
\newblock {\em The Journal of Machine Learning Research}, 18(1):6869--6898,
  2017.

\bibitem{jacobs1997bias}
Robert~A Jacobs.
\newblock Bias/variance analyses of mixtures-of-experts architectures.
\newblock {\em Neural computation}, 9(2):369--383, 1997.

\bibitem{jacobs1991adaptive}
Robert~A Jacobs, Michael~I Jordan, Steven~J Nowlan, and Geoffrey~E Hinton.
\newblock Adaptive mixtures of local experts.
\newblock {\em Neural Computation}, 3(1):79--87, 1991.

\bibitem{ke2017lightgbm}
Guolin Ke, Qi~Meng, Thomas Finley, Taifeng Wang, Wei Chen, Weidong Ma, Qiwei
  Ye, and Tie-Yan Liu.
\newblock Lightgbm: A highly efficient gradient boosting decision tree.
\newblock In {\em Advances in neural information processing systems}, pages
  3146--3154, 2017.

\bibitem{ke2018tabnn}
Guolin Ke, Jia Zhang, Zhenhui Xu, Jiang Bian, and Tie-Yan Liu.
\newblock Tabnn: A universal neural network solution for tabular data.
\newblock 2018.

\bibitem{li2016deep}
Yifeng Li, Chih-Yu Chen, and Wyeth~W Wasserman.
\newblock Deep feature selection: theory and application to identify enhancers
  and promoters.
\newblock {\em Journal of Computational Biology}, 23(5):322--336, 2016.

\bibitem{meir2000localized}
Ron Meir, Ran El-Yaniv, and Shai Ben-David.
\newblock Localized boosting.
\newblock In {\em COLT}, pages 190--199. Citeseer, 2000.

\bibitem{popov2019neural}
Sergei Popov, Stanislav Morozov, and Artem Babenko.
\newblock Neural oblivious decision ensembles for deep learning on tabular
  data.
\newblock {\em arXiv preprint arXiv:1909.06312}, 2019.

\bibitem{prokhorenkova2018catboost}
Liudmila Prokhorenkova, Gleb Gusev, Aleksandr Vorobev, Anna~Veronika Dorogush,
  and Andrey Gulin.
\newblock Catboost: unbiased boosting with categorical features.
\newblock In {\em Advances in neural information processing systems}, pages
  6638--6648, 2018.

\bibitem{quinlan1979discovering}
J~Ross Quinlan.
\newblock Discovering rules by induction from large collections of examples.
\newblock {\em Expert Systems in the Micro electronics Age}, 1979.

\bibitem{seyedhosseini2015disjunctive}
Mojtaba Seyedhosseini and Tolga Tasdizen.
\newblock Disjunctive normal random forests.
\newblock {\em Pattern Recognition}, 48(3):976--983, 2015.

\bibitem{shalev2014understanding}
Shai Shalev-Shwartz and Shai Ben-David.
\newblock {\em Understanding machine learning: From theory to algorithms}.
\newblock Cambridge university press, 2014.

\bibitem{shavitt2018regularization}
Ira Shavitt and Eran Segal.
\newblock Regularization learning networks: Deep learning for tabular datasets.
\newblock In {\em Advances in Neural Information Processing Systems}, pages
  1386--1396, 2018.

\bibitem{Simon}
Hans~Ulrich Simon.
\newblock On the number of examples and stages needed for learning decision
  trees.
\newblock In {\em Proceedings of the Third Annual Workshop on Computational
  Learning Theory}, COLT ’90, page 303–313, San Francisco, CA, USA, 1990.
  Morgan Kaufmann Publishers Inc.

\bibitem{vanschoren2014openml}
Joaquin Vanschoren, Jan~N Van~Rijn, Bernd Bischl, and Luis Torgo.
\newblock Openml: networked science in machine learning.
\newblock {\em ACM SIGKDD Explorations Newsletter}, 15(2):49--60, 2014.

\bibitem{yang2018deep}
Yongxin Yang, Irene~Garcia Morillo, and Timothy~M Hospedales.
\newblock Deep neural decision trees.
\newblock {\em arXiv preprint arXiv:1806.06988}, 2018.

\bibitem{yoon2018invase}
Jinsung Yoon, James Jordon, and Mihaela van~der Schaar.
\newblock Invase: Instance-wise variable selection using neural networks.
\newblock 2018.

\bibitem{zou2005regularization}
Hui Zou and Trevor Hastie.
\newblock Regularization and variable selection via the elastic net.
\newblock {\em Journal of the royal statistical society: series B (statistical
  methodology)}, 67(2):301--320, 2005.

\end{thebibliography}

\appendix

\newpage

\section*{\Large Apendices}

\section{OR and AND Gates}
\label{sec:orand}
The (soft) neural OR and AND gates were defined as
\begin{eqnarray*}
\OR(\bx) &\eqdef& \tanh\paren{\sum_{i=1}^{d}\bx_i + d - 1.5 }, \hspace{50pt}
\AND(\bx) \eqdef \tanh\paren{\sum_{i=1}^{d}\bx_i - d + 1.5 }.
\end{eqnarray*}
By replacing the $\tanh$ activation with a $\sign$ activation, and setting the 
bias term to 1 (instead of 1.5), we obtain exact 
binary gates,
\begin{eqnarray*}
\OR(\bx) &\eqdef& \sign\paren{\sum_{i=1}^{d}\bx_i + d - 1 }, \hspace{50pt}
\AND(\bx) \eqdef \sign\paren{\sum_{i=1}^{d}\bx_i - d + 1 }.
\end{eqnarray*}
Consider a binary vector $\bx \in \{\pm 1\}^d$. We prove that 
$$
\AND(\bx) \equiv \bigwedge_{i=1}^d \bx_i,
$$
where, in the definition of the logical ``and'', $-1$ is equivalent to 0.
If for any $1 \leq i \leq d$, $\bx_i = 1$, then $\wedge_{i=1}^d \bx_i = 1$.
Conversely, we have,
$$
\AND(\bx) = \sum_{i=1}^{d}\bx_i - d + 1 = d - d + 1 = 1,
$$
and the application of the $\sign$ activation yields 1. In the case of the soft neural AND gate, we get $tanh(1) \approx 0.76$; therefore, we set the bias term to 1.5 to get an output closer to 1 ($tanh(1.5) \approx 0.9$). 

Otherwise, there exists at least one index $1 \leq j \leq d$, such that $\bx_j = -1$,
and $\wedge_{i=1}^d \bx_i = -1$.
In this case,
$$
\AND(\bx) = \sum_{i=1}^{d}\bx_i - d + 1 = \bx_j + \sum_{i \neq j} \bx_i -d +1 \leq -1 + (d-1) -d + 1 = -1,
$$
and by applying the $\sign$ activation we obtain $-1$.
This proves that the $\AND(\bx)$ neuron is equivalent to a logical ``AND'' gate 
in the binary case. A very similar proof shows that
$$
\OR(\bx) \equiv \bigvee_{i=1}^d \bx_i.
$$

\section{Proof of Theorem 2}
\label{sec:proof_theorem1}
We bound the VC-dimension of a DNF formula in two steps. First, we derive an upper bound on the VC-dimension 
of a single conjunction, and then extend it to a disjunction of $k$ conjunctions.
We use the following simple lemma.
\begin{lemma}
\label{lemma:inclusion_of_hypotheses_set}
For every two hypothesis classes, $H' \subseteq H$, it  holds that $VCDim(H') \leq VCDim(H)$.
\end{lemma}
\begin{proof}
Let $d =VCDim(H')$. By definition, there exist $d$ points that can be shattered by $H'$. 
Therefore, there exist $2^d$ hypotheses $\{h'_i\}_{i=1}^{2^d}$ in $H'$, which shatter these points. 
By assumption, $\{h'_i\}_{i=1}^{2^d} \subseteq H$, so $VCDim(H) \geq d$.
\end{proof}

For any conjunction on $n$ Boolean variables (regardless of the number of literals), it is possible to construct an equivalent decision tree of rank 1. 
The construction is straightforward.
If $\bigwedge_{i=1}^{\ell} x_i$ is the conjunction,
the decision tree  consists of a single main branch of $\ell$ internal 
decision nodes connected sequentially.
Each left child in this tree corresponds to decision ``1'', and each right child corresponds to decision ``0''.  
The root is indexed 1 and contains the literal $x_1$.
For $1 \leq i < \ell$, internal node $i$ contains the decision literal
$x_i$ and its left child is node $i+1$  (whose decision literal is $x_{i+1}$). 
See the example in Figure~\ref{fig:rank_one_DT}. 


It follows that the hypothesis class of conjunctions is contained in the class of rank-$1$ decision trees.
Therefore, by Lemma~\ref{lemma:inclusion_of_hypotheses_set} and Theorem \ref{theorem:VCdimDT}, the 
VC-dimension of conjunctions is bounded above by $n+1$.

We now derive the upper bound on the VC-dimension of a disjunction of $k$ conjunctions. Let $C$ be the
class of conjunctions, and let $D_k(C)$
be the class of a disjunction of $k$ conjunctions.
Clearly, $D_k(C)$ is a $k$-fold union of the class $C$,
namely,
$$
D_k(C) = \left\{ 
\bigcup_{i=0}^k c_i \ | c_i \in C 
\right\}.
$$
By Lemma 3.2.3 in Blummer et al. \cite{blumer1989learnability}, if  $d = VCDim(C)$, then for all $k \geq 1$, $VCDim(D_k(C)) \leq 2dk\log(3k)$. Therefore, for the class $DNF_n^k$, of DNF formulas with $k$ conjunctions on $n$ Boolean variables, we have
\begin{equation*}
    VCDim(DNF_n^k) \leq 2(n+1)k\log(3k) .
\end{equation*}

\begin{figure}[h]
    \centering
    \includegraphics[width=0.48\textwidth]{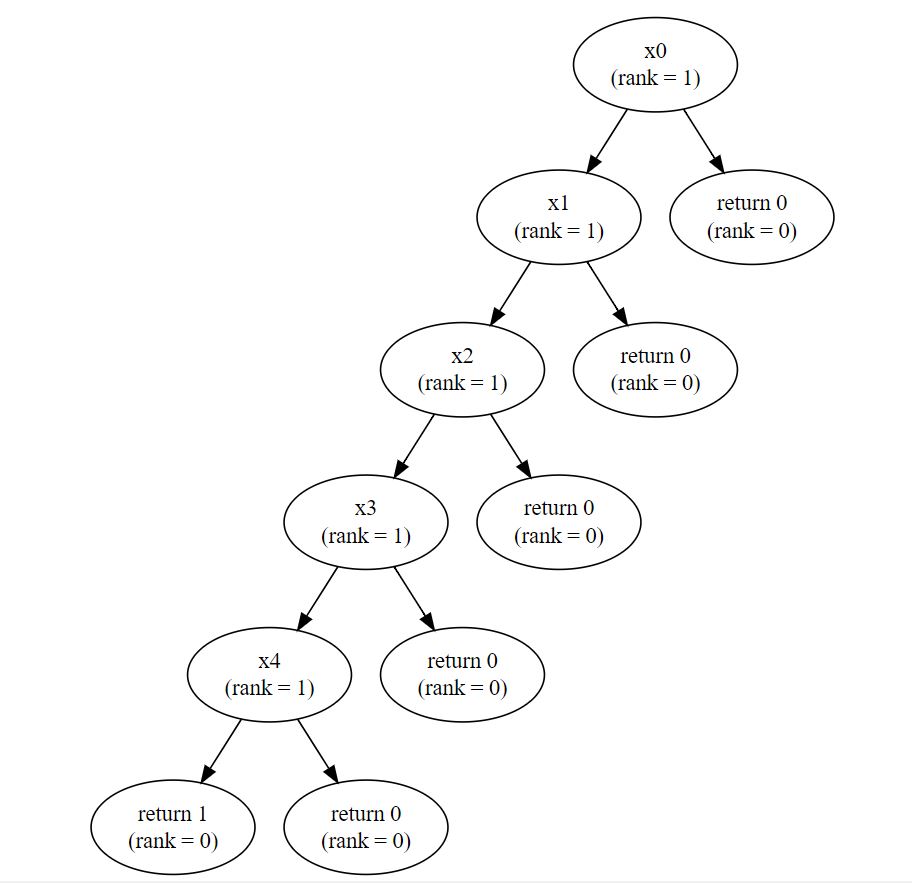}
    \caption{An example of a decision tree with rank 1, which is equivalent to the conjunction $x_0 \wedge x_1 \wedge x_2 \wedge x_3 \wedge x_4$. }
    \label{fig:rank_one_DT}
\end{figure}

\section{Tabular Dataset Description}
\label{sec:Appendix_datasets_description}
We use datasets (See Table 4) that differ in several aspects such as in the number of features (from 16 up to 200), the number of classes (from 2 up to 9), and the number of samples (from 10k up to 200k). To keep things simple, we selected datasets with no missing values, and that do not require preprocessing. All models were trained on the raw data without any feature or data engineering and without any kind of data balancing or weighting. Only feature standardization was applied.

\begin{table}[h!]
\begin{center}
\scalebox{0.7}{
 \begin{tabular}{||l c c c c c||} 
 \hline
 Dataset & features & classes & samples & source & link\\ [0.5ex] 
 \hline\hline

  Otto Group  & 93 & 9 & 61.9k & Kaggle & kaggle.com/c/otto-group-product-classification-challenge/overview \\ 
 \hline
  Gesture Phase  & 32 & 5 & 9.8k & OpenML & openml.org/d/4538 \\ 
 \hline
  Gas Concentrations & 129 & 6 & 13.9k & OpenML & openml.org/d/1477 \\ 
 \hline
  Eye Movements & 26 & 3 & 10.9k & OpenML & openml.org/d/1044 \\ 
 \hline
   Santander Transaction & 200 & 2 & 200k & Kaggle & kaggle.com/c/santander-customer-transaction-prediction/overview \\ 
 \hline
  House  & 16 & 2 & 22.7k & OpenML & openml.org/d/821 \\ 
 \hline

\end{tabular}}
\end{center}
\caption{A description of the tabular datasets}
\end{table}

\section{Experimental Protocol}
\label{appendix:experiments_details}

\subsection{Data Partition and Grid Search Procedure}
\label{appendix:data partition and grid search process appendix}
All experiments in our work, using both synthetic and real datasets, were done through a grid search process. Each dataset was first randomly divided into five folds in a way that preserved the original distribution. Then, based on these five folds, we created five partitions of the dataset as follows. Each fold is used as the test set in one of the partitions, while the other folds are used as the training and validation sets. This way, each partition was $20\%$ test, $10\%$ validation, and $70\%$ training. This division was done once \footnote{We used seed number 1.}, and the same partitions were used for all models.
Based on these partitions, the following grid search process was repeated three times with three different seeds\footnote{We used seed numbers 1, 2, 3.} (with the exact same five partitions as described before).

\begin{algorithm}[H]
\SetAlgoLined
\textbf{Input:} model, configurations\_list \\
results\_list = [ ] \\
\For{i=1 to n\_partitions}{
    val\_scores\_list = [ ] \\
    test\_scores\_list = [ ] \\
    train, val, test = read\_data(partition\_index=i) \\
     \For{c in configurations\_list}{
        trained\_model = model.train(train\_data=train, val\_data=val, configuration=c) \\
        trained\_model.load\_weights\_from\_best\_epoch() \\
        val\_score = trained\_model.predict(data=val) \\
        test\_score = trained\_model.predict(data=test) \\
        val\_scores\_list.append(val\_score) \\
        test\_scores\_list.append(test\_score) \\
     }
     best\_val\_index = get\_index\_of\_best\_val\_score(val\_scores\_list) \\
     test\_res = test\_scores\_list[best\_val\_index] \\
     results\_list.append(test\_res) \\
 }     
 mean = mean(results\_list) \\
 sem = standard\_error\_of\_the\_mean(results\_list)\\
 \textbf{Return:} mean, sem
 \caption{Grid Search Procedure}
 \label{alg:grid_search}
\end{algorithm}
The final mean and sem\footnote{For details, see: docs.scipy.org/doc/scipy/reference/generated/scipy.stats.sem.html} that we presents in all experiments are the average across the three seeds. 
Additionally, as can be seen from Algorithm \ref{alg:grid_search}, the model that was trained on the training set ($70\%$) is the one that is used to evaluate performance on the test set ($20\%$). This was done to keep things simple. The loading wights command is relevant for the neural network models. While for the XGBoost, the framework handles the optimal number of estimators on prediction time (accordingly to early stopping on training time). 

\subsection{Training Protocol}
\label{appendix:Training details}
The \DNFNET~and the FCN were implemented using Tesnorflow \cite{abadi2016tensorflow}. To make a fair comparison, for both models, we used the same batch size\footnote{For \DNFNET~, when using 3072 formulas, we set the batch size to 1024 on the Santander Transaction and Gas datasets and when using 2048 formulas, we set the batch size to 1024 on the Santander Transaction dataset. This was done due to memory issues.} of 2048, and the same learning rate scheduler (reduce on plateau) that monitors the training loss. We set a maximum of 1000 epochs and used the same early stopping protocol (30 epochs) that monitors the validation score. Moreover, for both of them, we used the same loss function (softmax-cross-entropy for multi-class datasets and sigmoid-cross-entropy for binary datasets) and the same optimizer (Adam with default parameters).

For \DNFNET~we used an initial learning rate of $0.05$. For FCN, we added the initial learning rate to the grid search with values of $\{0.05, 0.005, 0.0005\}$.

For XGBoost \cite{chen2016xgboost}, we set the maximal number of estimators to be 2500, and used an early stopping of 50 estimators that monitors the validation score.

All models were trained on GPUs - Titan Xp 12GB RAM.

Additionally, in the case of \DNFNET,  we took a symmetry-breaking approach between the different DNNFs. This is reflected by the \DNNF~group being divided equally into four subgroups where, for each subgroup, the number of conjunctions is equal to one of the following values $[6, 9, 12, 15]$, and the group of conjunctions of each \DNNF~was divided equally into three subgroups where, for each subgroup, the conjunction length is equal to one of the following values $[2, 4, 6]$.
The same approach was used for the parameter $p$ of the random mask. The \DNNF~group was divided equally into five subgroups where, for each subgroup, $p$ is equal to one of the following values $[0.1, 0.3, 0.5, 0.7, 0.9]$. In all experiments we used the same values.

\subsection{Grid Parameters -- Tabular Datasets}
\label{appendix:grid_params}
\subsubsection{DNF-Net}
    \begin{table}[h!]
    \begin{center}
    \scalebox{0.8}{
      \begin{tabular}{SS}
        \toprule
          \multicolumn{2}{c}{\textbf{DNF-Net (42 configs)}} \\
          \midrule
        \textbf{hyperparameter} & \textbf{values}     \\
        \hline
        \text{n. formulas} &                \text{$\{64, 128, 256, 512, 1024, 2048, 3072\}$} \\
        \text{feature selection beta} &     \text{$\{1.6 ,1.3, 1., 0.7, 0.4, 0.1\}$} \\
        \bottomrule
      \end{tabular}}
    \end{center}
    \end{table}
    \FloatBarrier
    
\subsubsection{XGBoost}
\begin{table}[h!]
\begin{center}
\scalebox{0.8}{
  \begin{tabular}{SS}
    \toprule
      \multicolumn{2}{c}{\textbf{XGBoost (864 configs)}} \\
      \midrule
    \textbf{hyperparameter} & \textbf{values}     \\
    \hline
    \text{n. estimators} &      \text{$\{2500\}$}  \\
    \text{learning rate} &      \text{$\{0.001, 0.005, 0.01, 0.05, 0.1, 0.5\}$}  \\
    \text{max depth} &          \text{$\{2, 3, 4, 5, 7, 9, 11, 13, 15\}$}  \\
    \text{colsample by tree} &  \text{$\{0.25, 0.5, 0.75, 1.\}$}  \\
    \text{sub sample} &         \text{$\{0.25, 0.5, 0.75, 1.\}$}  \\
    \bottomrule
  \end{tabular}}
\end{center}
\end{table}
\FloatBarrier

\subsubsection{Fully Connected Networks}
The FCN networks are constructed using Dense-RELU-Dropout blocks with $L_2$ regularization. The network's blocks are defined in the following way. Given depth and width parameters, we examine two different configurations: (1) the same width is used for the entire network (e.g., if the width is 512 and the depth is four, then the network blocks are [512, 512, 512, 512]), and (2) the width parameter defines the width of the first block, and the subsequent blocks are reduced by a factor of 2 (e.g., if the width is 512 and the depth is four, then the network blocks are [512, 256, 128, 64]). On top of the last block we add a simple linear layer that reduce the dimension into the output dimension. The dropout and $L_2$ values are the same for all blocks. 

\begin{table}[h!]
\begin{center}
\scalebox{0.8}{
  \begin{tabular}{SS}
    \toprule
      \multicolumn{2}{c}{\textbf{FCN (3300 configs)}} \\
      \midrule
    \textbf{hyperparameter} & \textbf{values}     \\
    \hline
    \text{depth} &                        \text{$\{1, 2, 3, 4, 5, 6\}$}  \\
    \text{width} &                        \text{$\{128, 256, 512, 1024, 2048\}$}  \\
    \text{$L_2$ lambda} &                   \text{$\{10^{-2}, 10^{-4}, 10^{-6}, 10^{-8}, 0.\}$}  \\
    \text{dropout} &                      \text{$\{0., 0.25, 0.5, 0.75\}$}  \\
    \text{initial learning rate} &        \text{$\{0.05, 0.005, 0.0005\}$}  \\
    \bottomrule
  \end{tabular}}
\end{center}
\end{table}
\FloatBarrier

\subsection{Ablation Study}
\label{appendix:Ablation Studies - technical details}
All ablation studies experiments were conducted using the grid search process as described in \ref{appendix:data partition and grid search process appendix}. In all experiments, we used the same training details as described on \ref{appendix:Training details} for \DNFNET. Where the only difference between the different experiments is the addition or removal of the components.

 The single hyperparameter that was fine-tuned using the grid search is the `feature selection beta' on the range $\{ 1.6, 1.3, 1., 0.7, 0.4, 0.1\}$, in experiments in which the feature selection component is involved. In the other cases, only one configuration was tested in the grid search process for a specific number of formulas.

\subsection{Feature Selection Analysis}
\label{appendix:Feature selection analysis}
The input features $\bx \in \reals^d$ of all six datasets were generated from a $d$-dimensional Gaussian distribution with no correlation across the features, $\bx \sim \mathbb{N}(0,I)$. The label $\by$ is sampled as a Bernoulli random variable with $\mathbb{P}(\by = 1 | \bx) = \frac{1}{1+logit(\bx)}$, where $logit(\bx)$ is varied to create the different synthetic datasets ($\bx_i$ refers to the $i$th entry):
\begin{enumerate}
    \item \textbf{Syn1}: $logit(\bx) = exp(\bx_1\bx_2)$
    \item \textbf{Syn2}: $logit(\bx) = exp(\sum_{i=3}^6 \bx_i^2-4)$
    \item \textbf{Syn3}: $logit(\bx) = -10\sin(2\bx_7)+2|\bx_8|+\bx_9+exp(-\bx_{10}) - 2.4$
    \item \textbf{Syn4}: if $\bx_{11} < 0$, logit follows \textbf{Syn1}, else, logit follows \textbf{Syn2}
    \item \textbf{Syn5}: if $\bx_{11} < 0$, logit follows \textbf{Syn1}, else, logit follows \textbf{Syn3}
    \item \textbf{Syn6}: if $\bx_{11} < 0$, logit follows \textbf{Syn2}, else, logit follows \textbf{Syn3}
\end{enumerate}
We compare the performance of a basic FCN on three different cases: (1) {\bf oracle (ideal) feature selection} -- where the input feature vector is multiplied element-wise with an input oracle mask, whose $i$th entry equals 1 iff the $i$th feature is relevant (e.g., on Syn1, features 1 and 2 are relevant, and on Syn4, features 1-6, and 11 are relevant), (2) {\bf our (learned) feature selection mask} -- where the input feature vector is multiplied element-wise with the mask $\bm_t$, i.e.,  the entries of the mask $\bm_s$ (see Section~\ref{sec:feature_selection}) 
are all fixed to 1, and (3) {\bf no feature selection}.

From each dataset, we generated seven different instances that differ in their input size, \\ $d \in [11, 50, 100, 150, 200, 250, 300]$. Where when the input dimension $d$ increases, the same logit function is used. Each instance contains 10k samples that were partitioned as described in Section \ref{appendix:data partition and grid search process appendix}. We treated each instance as an independent dataset, and the grid search process that is described in Section \ref{appendix:data partition and grid search process appendix} was done for each one.

The FCN that we used has two dense hidden layers [64, 32] with a RELU activation. To keep things simple, we have not used drouput or any kind of regularization. 
The same training protocol was used for all three models. We used the same learning rate scheduler, early stopping protocol, loss function and optimizer as appear in Section \ref{appendix:Training details}\footnote{We noticed that in this scenario, a large learning rate or large batch size leads to a decline in the performance of the 'FCN with the feature selection'. While the simple FCN and the 'FCN with oracle mask' remains approximately the same.}. We use a batch size of 256, and an initial learning rate of 0.001. The only hyperparameter that was fine-tuned is the `feature selection beta' in the case of `FCN with feature selection' on the range $\{1.3, 1., 0.7, 0.4\}$. For the two other models, only a single configuration was tested in the grid search process.

\end{document}